\documentclass[11 pt]{article}
\usepackage{authblk}
\usepackage[margin=1in]{geometry}
\usepackage[utf8x]{inputenc} 
\usepackage[T1]{fontenc}    
\usepackage{hyperref}       
\usepackage{url}            
\usepackage{booktabs}       
\usepackage{amsfonts}       
\usepackage{nicefrac}       
\usepackage{microtype}      

\usepackage{graphicx}
\usepackage{subfigure}
\usepackage{booktabs} 
\usepackage{bm} 
\usepackage{bbm}
\usepackage{nicefrac}       
\usepackage{amsmath}
\usepackage{amsthm}
\usepackage{amssymb} 
\usepackage{hhline}
\usepackage{blindtext}
\usepackage{multicol}
\usepackage[dvipsnames]{xcolor}
\usepackage{xcolor}
\usepackage[ruled,vlined]{algorithm2e}
\usepackage{natbib}
\usepackage{subfigure}

\usepackage{pgffor}

\newcommand{\stdev}[1]{{\tiny (#1)}}

\newcommand{\z}{\mathbf{z}}

\newcommand{\x}{\mathbf{x}}

\newcommand{\knn}{$k-$NN}

\newtheorem*{mylem*}{Lemma} 
\newtheorem{myprop}{Proposition} 
 
\newtheorem*{mythe*}{Theorem} 
 
\newtheorem*{mycor*}{Corollary} 

\newtheorem*{sec_proof*}{Proof}
\newtheorem*{sec_prop*}{Proposition}
\newtheorem*{sec_lemma*}{Lemma}

\newcommand{\NOcause}[1]{}

    \newtheoremstyle{TheoremNum}
        {\topsep}{\topsep}              
        {\itshape}                      
        {}                              
        {\bfseries}                     
        {.}                             
        { }                             
        {\thmname{#1}\thmnote{ \bfseries #3}}
    \theoremstyle{TheoremNum}

    \newtheoremstyle{TheoremNum}
        {\topsep}{\topsep}              
        {\itshape}                      
        {}                              
        {\bfseries}                     
        {.}                             
        { }                             
        {\thmname{#1}\thmnote{ \bfseries #3}}
    \theoremstyle{TheoremNum}

\title{An Adjusted Nearest Neighbor Algorithm Maximizing the F-Measure from Imbalanced Data }

\author{ R\'emi Viola$^{1,2}$, R\'emi Emonet$^{1}$, Amaury Habrard$^{1}$, Guillaume Metzler$^{1}$, S\'ebastien Riou$^{2}$ and Marc Sebban}
\affil[1]{Laboratoire Hubert Curien UMR 5516,
Univ Lyon, UJM F-42023,  Saint-Etienne, France. }
\affil[2]{Direction G\'en\'erale des Finances Publiques, Minist\`ere de l'Economie et des Finances, Paris, France. }

\date \today

\begin{document}
\maketitle

\begin{abstract}
In this paper, we address the challenging problem of learning from imbalanced data using a Nearest-Neighbor (NN) algorithm. 
In this setting, the minority examples typically belong to the class of interest  requiring the optimization of specific criteria, like the F-Measure. Based on simple geometrical ideas,  we introduce an algorithm that reweights the distance between a query sample and any positive training example. This leads to a modification of the Voronoi regions and thus of the decision boundaries of the NN algorithm.  We provide a theoretical justification about the weighting scheme needed to reduce the False Negative rate while controlling the number of False Positives. We perform an extensive experimental study on many public imbalanced datasets, but also on large scale non public data from the French Ministry of Economy and Finance on a tax fraud detection task, showing that our method is very effective and, interestingly, yields the best performance when combined with state of the art sampling methods.
\end{abstract}

\section{Introduction}
\label{sec:intro}

Intrusion detection, health care insurance or bank fraud identification, and more generally anomaly detection, {\it e.g.} in medicine or in industrial processes, are tasks requiring to address the challenging problem of learning from imbalanced data~\cite{aggarwal2017outlier,chandola2009anomaly,bauder2018data}.
In such a setting, the training set is composed of a few positive examples ({\it e.g.} the frauds) and a huge amount of negative samples ({\it e.g.} the genuine transactions). Standard  learning algorithms struggle to deal with this imbalance scenario because they are typically based on the minimization of (a surrogate of) the 0-1 loss. Therefore, a trivial solution  consists in assigning the majority label to any test query leading to a high performance from an accuracy perspective but missing the (positive) examples of interest. 
To overcome this issue, several strategies have been developed over the years. 
The first one consists in the optimization of loss functions based on measures that are more appropriate for this context such as the \textit{Area Under the ROC Curve} (AUC), the \textit{Average Precision} (AP) or the \textit{F-measure} to cite a few~\cite{ferri2009experimental,steck2007hinge}.
The main pitfalls related to such a strategy concern the difficulty to directly optimize non smooth, non separable and non convex measures.
A simple and usual solution to fix this problem consists in using off-the-shelf learning algorithms (maximizing the accuracy) and a posteriori pick the model with the highest AP or F-Measure.
Unfortunately, this might be often suboptimal.
A more elaborate solution aims at designing differentiable versions of the previous non-smooth measures and optimizing them, {\it e.g.} as done by gradient boosting in~\cite{FreryHSCH17} with a smooth surrogate of the Mean-AP.
The second family of methods is based on the modification of the distribution of the training data using sampling strategies~\cite{fernandez2018smote}. 
This is typically achieved by removing examples from the majority class, as done, {\it e.g.}, in \textit{ENN} or \textit{Tomek's Link}~\cite{tomek1976}, and/or by adding  examples from the minority class, as in SMOTE~\cite{chawla2002smote} and its variants, or by resorting to generative adversarial models~\cite{GoodfellowPMXWOCB14}. 
One peculiarity of imbalanced datasets can be interpreted from a geometric perspective.
As illustrated in Fig.~\ref{fig:voronoi} (left) which shows the Voronoi cells on an artificial imbalanced dataset (where two adjacent cells have been merged if they concern examples of the same class), the regions of influence of the positive examples are much smaller than that of the negatives.
This explains why at test time, in imbalanced learning, the risk to get a false negative is high, leading to a low F-measure, the criterion we focus on in this paper, defined as the harmonic mean of the $Precision=\frac{TP}{TP+FP}$ and the $Recall=\frac{TP}{TP+FN}$, where $FP$ (resp. $FN$) is the number of false positives (resp. negatives) and $TP$ the number of true positives.
Note that increasing the regions of influence of the positives would allow us to reduce $FN$ and improve the F-measure.
However, not controlling the expansion of these regions may have a dramatic impact on $FP$, and so on the F-Measure, as illustrated in Fig.~\ref{fig:voronoi} (right). 

\begin{figure}[t]
\includegraphics[width=.33\linewidth]{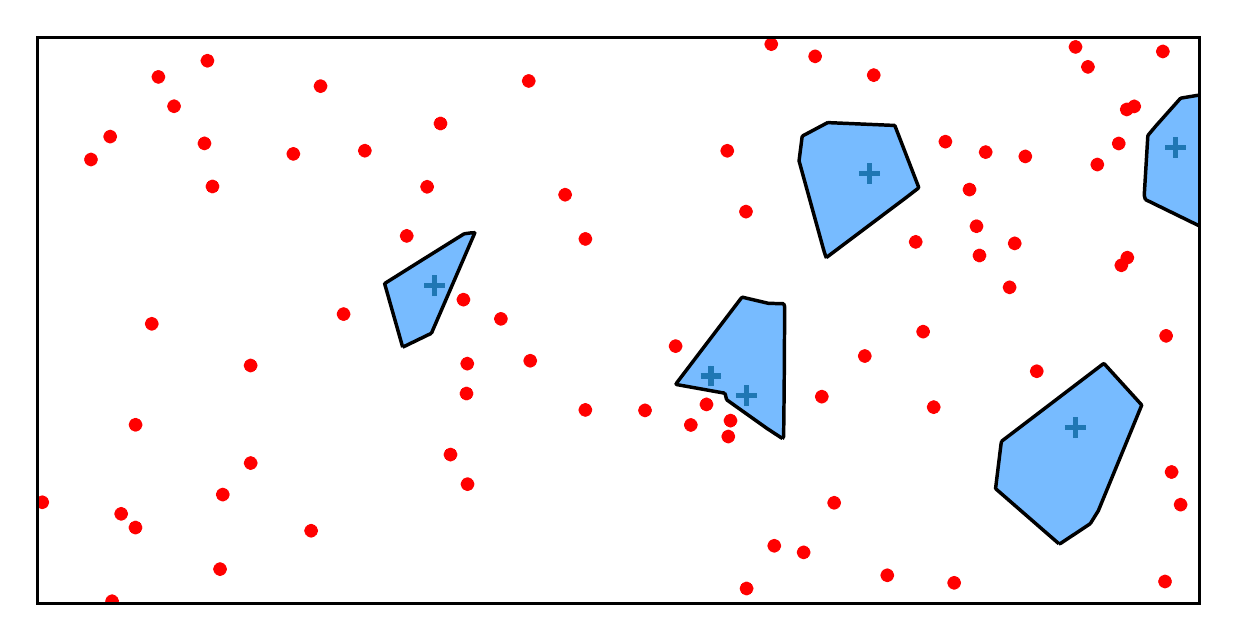}%
\includegraphics[width=.33\linewidth]{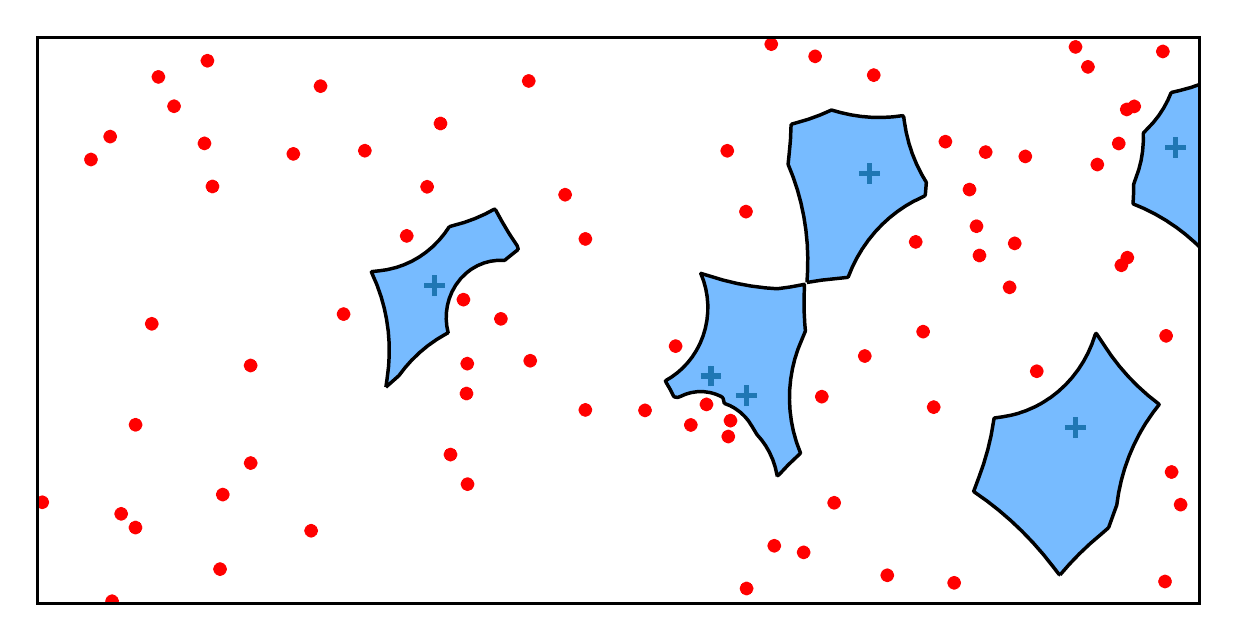}%
\includegraphics[width=.33\linewidth]{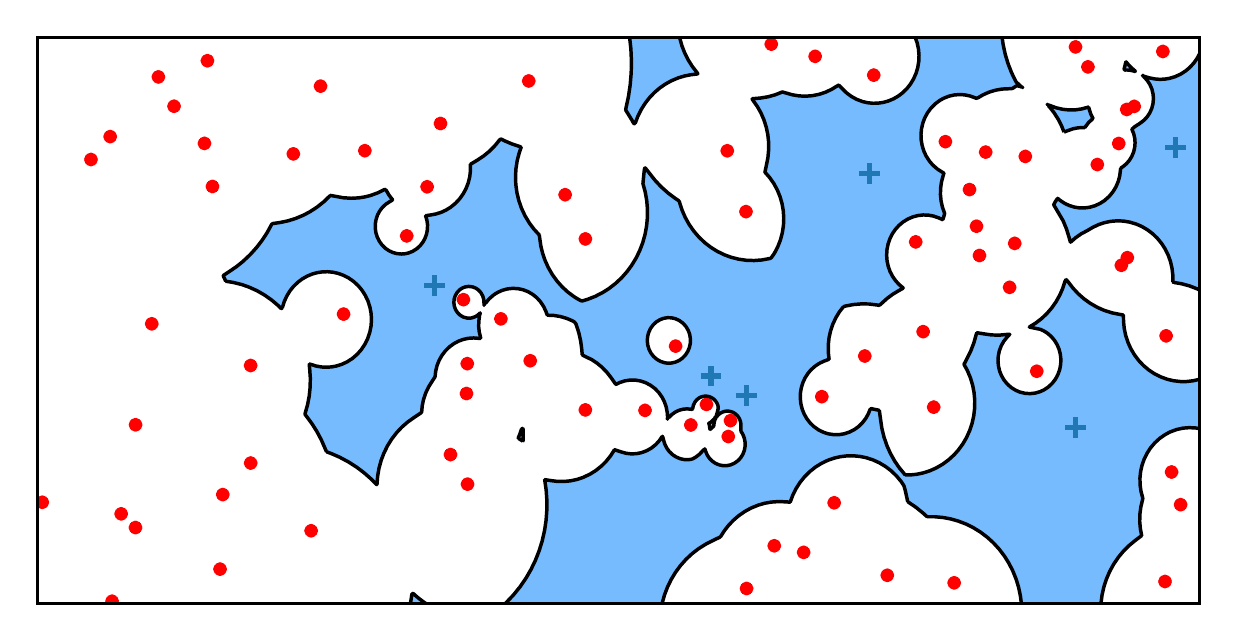}
\caption{Toy imbalanced dataset: On the left, the Voronoi regions around the positives are small.
  The risk to generate false negatives (FN) at test time is large.
  On the right: by increasing too much the regions of influence of the positives, the probability to get false positives (FP) grows.
  In the middle: an appropriate trade-off between the two previous situations.}
\label{fig:voronoi}
\end{figure}

The main contribution of this paper is about the problem of finding the appropriate trade-off (Fig.~\ref{fig:voronoi} (middle)) between the two above-mentioned extreme situations (large $FP$ or $FN$) both leading to a low F-Measure.
A natural way to increase the influence of positives may consist in using generative models (like GANs~\cite{GoodfellowPMXWOCB14}) to sample new artificial examples, mimicking the negative training samples.
However, beyond the issues related to the parameter tuning, the computation burden and the complexity of such a method, using GANs to optimize the precision and recall is still an open problem (see \cite{NIPS2018_7769} for a recent paper on this topic).
We show in this paper that a much simpler strategy can be used by modifying the distance exploited in a $k$-nearest neighbor (NN) algorithm~\cite{cover1967nearest} which enjoys many interesting advantages, including its simplicity, its capacity to approximate asymptotically any locally regular density, and its theoretical rootedness~\cite{luxburg2004distance,kontorovich2015bayes,kontorovich2016active}.
$k$-NN also benefited from many algorithmic advances during the past decade in the field of metric learning, aiming at optimizing under constraints the parameters of a metric, typically the Mahalanobis distance, as done in LMNN~\cite{weinberger2009distance} or ITML~\cite{davis-et-al-icml-2007} (see \cite{2015BelletHS} for a survey).
Unfortunately, existing metric learning methods are dedicated to enhance the $k$-NN accuracy and do not focus on the optimization of criteria, like the F-measure, in scenarios where the positive training examples are scarce.
A geometric solution to increase, at a very low cost, the region of influence of the minority class consists in modifying the distance when comparing a query example to a positive training sample.
More formally, we show in this paper that the optimization of the F-Measure is facilitated by weighting the distance to any positive by a coefficient $\gamma \in [0,1]$ leading to the expansion of the Voronoi cells around the minority examples.
An illustration is given in Fig.\ref{fig:voronoi} (middle) which might be seen as a good compromise that results in the reduction of $FN$ while controlling the risk to increase $FP$.
Note that our strategy boils down to  modifying the local density of the positive examples.
For this reason, we claim that it can be efficiently combined with SMOTE-based sampling methods whose goal is complementary and consists in generating examples on the path linking two (potentially far) positive neighbors.
Our experiments will confirm this intuition.

The rest of this paper is organized as follows. Section~\ref{sec:notations} is dedicated to the introduction of our notations. The related work is presented in Section~\ref{sec:related}. Section~\ref{sec:ourmethod} is devoted to the presentation of our method. We perform an extensive experimental study in Section~\ref{sec:experiments} on many imbalanced datasets, including non public data from the French Ministry of Economy and Finance on a tax fraud detection task. We give evidence about the complementarity of our method with sampling strategies. We finally conclude in Section~\ref{sec:conclusion}.


\section{Notations and Evaluation Measures}
\label{sec:notations}

We consider a training sample $S=\{(\x_i,y_i), i=1,...,m\}$ of size $m$, drawn from an unknown joint distribution $\mathcal{Z}=\mathcal{X}\times \mathcal{Y}$, where $\mathcal{X}= \mathbb{R}^p$  is the feature space and $\mathcal{Y}= \{-1,1\}$ is the set of labels. 
Let us assume that  $S=S_+ \cup S_-$ with $m_+$ positives $\in S_+$ and $m_-$ negatives  $\in S_-$ where $m=m_++m_-$. 

Learning from imbalanced datasets requires to optimize appropriate measures that take into account the scarcity of positive examples. Two measures are usually used: the \textit{Recall} or \textit{True Positive Rate} which measures the capacity of the model to recall/detect positive examples, and the \textit{Precision} which is the confidence in the prediction of a positive label:
\[\text{Recall}=\dfrac{\text{$TP$}}{\text{$TP$}+\text{$FN$}} \quad \text{and} \quad \text{Precision}=\dfrac{\text{$TP$}}{\text{$TP$}+\text{$FP$}},\]
where $FP$ (resp. $FN$) is the number of false positives (resp. negatives) and $TP$ is the number of true positives. Since one can arbitrarily improve the Precision if there is no constraint on the Recall (and vice-versa), they are usually combined into a single measure: the \textit{F-measure}~\cite{rijsbergen1979information} (or $F_1$ score), which is widely used in fraud and anomaly detection, and more generally in imbalanced classification~\cite{gee2014fraud}.
 \[F_1=\dfrac{2\times \text{Precision}\times \text{Recall}}{\text{Precision} + \text{Recall}}=\dfrac{2 \text{TP}}{2\text{TP}+\text{FN}+\text{FP}}.\]
Note that $F_1$ considers the Precision and Recall equally. 
%


\section{Related Work}
\label{sec:related}

In this section, we present the main strategies that have been proposed in the literature to address the problem of learning from imbalanced datasets. We first present methods specifically dedicated to enhance a $k$-NN classifier. Then, we give an overview of the main sampling strategies used to balance the classes. All these methods will be used in the experimental comparison in Section~\ref{sec:experiments}.

\subsection{Distance-based Methods}
\label{subsec:distance}

Several strategies have been devised to improve $k$-NN. The oldest method is certainly the one presented in~\cite{dudani1976distance} which consists in associating to each neighbor a voting weight that is inversely proportional to its distance to a query point $\x$. The assigned label $\hat{y}$ of $\x$ is  defined as:
\[\hat{y}=\sum_{\x_i\in \text{kNN}(\x)} y_i\times \frac{1}{d(\x,\x_i)},\]
where $\text{kNN}(\x)$ stands for the set of the $k$ nearest neighbors of $\x$.
A more refined version consists in taking into account both the distances to the nearest neighbors and the distribution of the features according to the class $p(\x_i\mid y_i)$~\cite{liu2011class}. 
Despite these modifications in the decision rule, the sparsity of the positives remains problematic and it possible that no positives belong in the neighborhood of a new query $\x$.
To tackle this issue, a version of the $k$-NN, called $kP$NN~\cite{zhang2013positive}, is to consider the region of the space around a new query $\x$ which contains exactly $k$ positive examples. By doing so, the authors are able to use the density of the positives to estimate the probability of belonging in the minority class.

A more recent version has been shown to perform better than the two previously mentioned: $kR$NN~\cite{zhang2017krnn}. If the idea remains similar (i.e. estimating the local sparsity of minority examples around a new query), the posterior probability of belonging in the minority class is adjusted so that it takes both the local and global disequilibrium for the estimation.

In order to weight the examples, in~\cite{hajizadeh2014nearest}, the authors use an iterative procedure to optimize the accuracy on each class using the nearest neighbor classifier (i.e. $k=1$) 

In~\cite{barandela2003strategies}, the authors account both the label and the distance to the neighbors $(\x_i,y_i)$ to define a weighted metric $d'$ from the euclidean distance $d$, as follows: 

\[d'(\x,\x_i)= \left(\frac{m_i}{m}\right)^{1/p} d(\x,\x_i),\]
where $m_i$ is the number of examples in the class $y_i$.
As we will see later, this method falls in the same family of strategies as our contribution, aiming at weighting the distance to the examples according to their label.
However, three main differences justify why our method will be better in the experiments:
(i) $d'$ is fixed in advance while we will adapt the weight that optimizes the $F$- measure;
(ii) because of (i), $d'$ needs to take into account the dimension $p$ of the feature space (and so will tend to $d$ as $p$ grows) while this will be intrinsically captured in our method by optimizing the weight given the $p$-dimensional space;
(iii) $d'$ is useless when combined with sampling strategies (indeed, $\frac{m_i}{m}$ would tend to be uniform) while our method will allow us to weight differently the original positive examples and the ones artificially generated.

Another way to assign weights to each class, which is close to the sampling methods presented in the next section, is to duplicate the positive examples according to the Imbalance Ratio: $m_-/m_+.$
Thus, it can be seen as a \textit{uniform} over-sampling technique, where all positives are replicated the same number of times.
However, note that this method requires to work with $k>1$.

A last family of methods that try to improve $k$-NN is related to  {\it metric learning}. LMNN~\cite{weinberger2009distance} or ITML~\cite{davis-et-al-icml-2007} are two famous examples which optimize under constraints a Mahalanobis distance $ d_{\mathbf{M}}(\x,\x_i)=\sqrt{(\x-\x_i)^{\top} \mathbf{M} (\x-\x_i)}$ parameterized by a positive semidefinite (PSD) matrix $\mathbf{M}$. Such methods seek a linear projection of the data in a latent space where the Euclidean distance is applied. As we will see in the following, our weighting method is a specific case of metric learning which looks for a diagonal matrix - applied only when comparing a query to a positive example - and that behaves well in terms of F-Measure.


\subsection{Sampling Strategies}
\label{subsec:sampling}

One way to overcome the issues induced by the lack of positive examples is to compensate artificially the imbalance between the two classes. Sampling strategies~\cite{fernandez2018smote} have been proven to be very efficient to address this problem. 
In the following, we overview the  most used methods in the literature. 

The Synthetic Minority Over-sampling Technique~\cite{chawla2002smote} (SMOTE) over-samples a dataset by creating new synthetic positive data. 
For each minority example $\x$, it randomly selects one of its $k$ nearest positive neighbors and then creates a new random positive point on the line between this neighbor and $\x$. This is done until some desired ratio is reached.

Borderline-SMOTE~\cite{han2005borderline} is an improvement of the SMOTE algorithm.
While the latter generates synthetic points from all positive points, BorderLine-SMOTE only focuses on those having more negatives than positives in their neighborhood. More precisely, new points are generated if the number $n$ of negatives in the $k$-neighborhood is such that $k/2 \leq n \le k$.

The Adaptive Synthetic~\cite{he2008adasyn} (ADASYN) sampling approach is also inspired from SMOTE. 
By using a weighted distribution, it gives more importance to classes that are more difficult to classify, {\it i.e.} where positives are surrounded by many negatives, and thus generates more synthetic data for these classes.

Two other strategies combine an over-sampling step with an under-sampling procedure. 
The first one uses the Edited Nearest Neighbors~\cite{wilson1972asymptotic} (ENN) algorithm on the top of SMOTE. 
After SMOTE has generated data, the ENN algorithm removes data that are misclassified by their $k$ nearest neighbors. The second one combines SMOTE with Tomek's link~\cite{tomek1976}. 
A Tomek's link is a pair of points $(\x_i, \x_j)$ from different classes for which there is no other point $\x_k$ verifying $d(\x_i,\x_k) \le d(\x_i,\x_j)$ or $d(\x_k,\x_j) \le d(\x_i,\x_j)$. 
In other words, $\x_i$ is the nearest neighbor of $\x_j$ and vice-versa. If so, one removes the example of $(\x_i, \x_j)$ that belongs to the majority class. Note both strategies tend to eliminate the overlapping between classes.

Interestingly, we can note that all the previous sampling methods try to overcome the problem of learning from imbalanced data by resorting to the notion of $k$-neighborhood.
This is justified by the fact that $k$-NN has been shown to be a good estimate of the density at a given point in the feature space.
In our contribution, we stay in this line of research.
Rather than generating new examples, that would have a negative impact from a complexity perspective, we locally modify the density around the positive points.
This is achieved by rescaling the distance between a test sample and the positive training examples.
We will show that such a strategy can be efficiently combined with sampling methods, whose goal is complementary, by potentially generating new examples in regions of the space where the minority class is not present.

\section{Proposed Approach}
\label{sec:ourmethod}

In this section, we present our $\gamma$\knn{} method which works by scaling the distance between a query point and positive training examples by a factor.

\subsection{A Corrected $k-$NN algorithm}
\label{subsec:algo}

Statistically, when learning from imbalanced data, a new query $\x$ has more chance to be close to a negative example due to the rarity of positives in the training set, even around the mode of the positive distribution.
We have seen two families of approaches that can be used to counteract this effect:
(i) creating new synthetic positive examples, and
(ii) changing the distance according to the class. 
The approach we propose falls into the second category.

We propose to modify how the distance to the positive examples is computed, in order to compensate for the imbalance in the dataset. 
We artificially bring a new query $\x$ closer to any positive data point $\x_i\in S_+$ in order to increase the effective area of influence of positive examples.
The new measure $d_\gamma$ that we propose is defined, using an underlying distance $d$ (e.g. the euclidean distance) as follows:
\[d_\gamma(\x,\x_i)=
\begin{cases}
d(\x,\x_i) & \text{if} \; \x_i\in S_-,\\
\gamma \cdot d(\x,\x_i) & \text{if} \;\x_i\in S_+.
\end{cases}  \]
As we will tune the $\gamma$ parameter, this new way to compute the similarity to a positive example is close to a Mahalanobis-distance learning algorithm, looking for a PSD matrix, as previously described.
However, the matrix $\mathbf{M}$ is restricted to be $\gamma^2 \cdot \mathbf{I}$, where $\mathbf{I}$ refers to the identity matrix.
Moreover, while metric learning typically works by optimizing a convex loss function under constraints, our $\gamma$ is simply tuned such as maximizing the non convex F-Measure.
Lastly, and most importantly, it is applied only when comparing the query to positive examples.
As such, $d_\gamma$ is not a proper distance, however, it is exactly this which allows it to compensate for the class imbalance.
In the binary setting, there is no need to have a $\gamma$ parameter for the negative class, since only the relative distances are used.
In the multi-class setting with $K$ classes, we would have to tune up to $K-1$ values of $\gamma$.

Before formalizing the $\gamma$\knn{} algorithm that will leverage the distance $d_\gamma$, we illustrate in Fig.~\ref{method:fig:twopoints}, on 2D data, the decision boundary induced by a nearest neighbor binary classifier that uses $d_\gamma$. We consider an elementary dataset with only two points, one positive and one negative. The case of $\gamma=1$, which is a traditional 1-NN is shown in a thick black line.
Lowering the value of $\gamma$ below $1$ brings the decision boundary closer to the negative point, and eventually tends to surround it very closely.
In Fig~\ref{intro:fig:simuldata}, two more complex datasets are shown, each with two positive points and several negative examples.
As intuited, we see that the $\gamma$ parameter allows to control how much we want to push the boundary towards negative examples.

\begin{figure}[t]
\centerline{\includegraphics[width=0.33\textwidth]{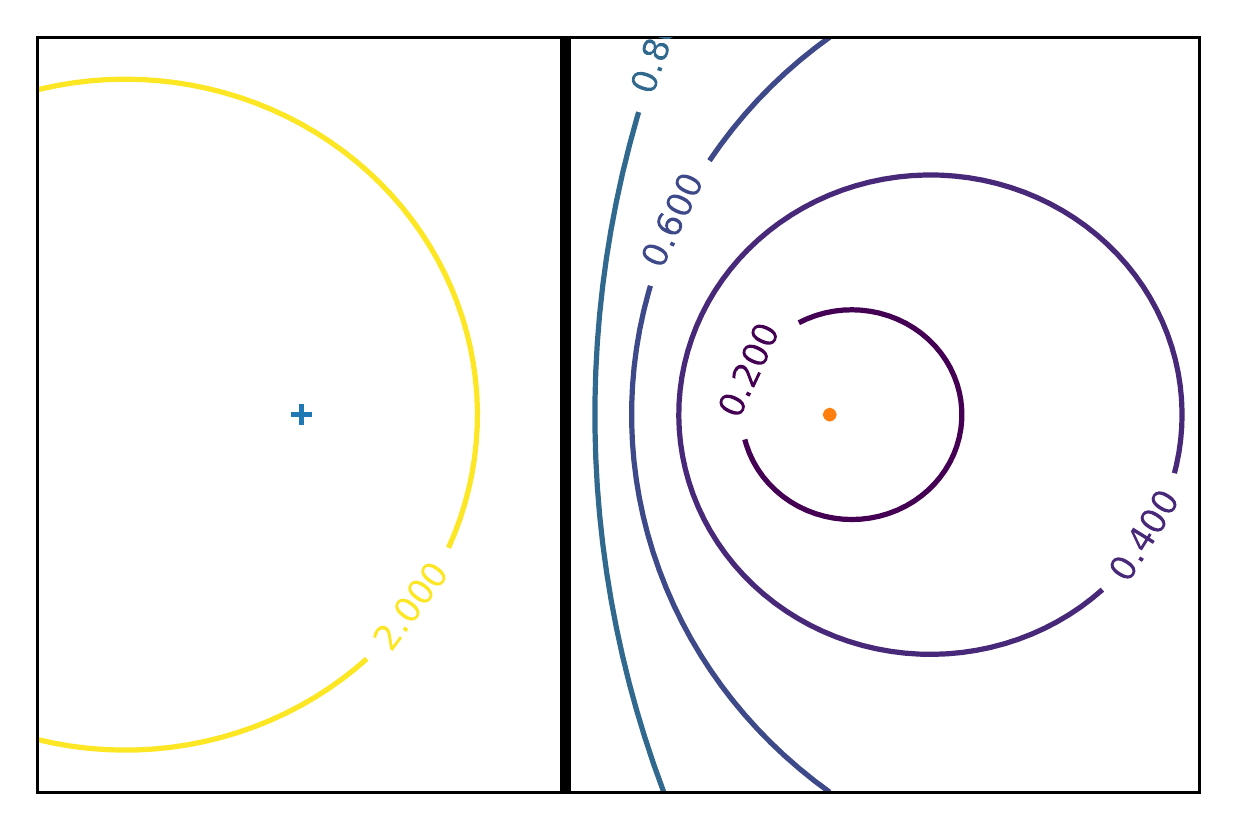}}
\caption{Evolution of the decision boundary based on $d_\gamma$, for a 1-NN classifier, on a 2D dataset with one positive (resp. negative) instance represented by a blue cross (resp. orange point).
  The value of $\gamma$ is given on each boundary ($\gamma=1$ on the thick line).
  }
 
\label{method:fig:twopoints}
\end{figure}

We can now introduce the $\gamma$\knn{} algorithm (see Algo~\ref{algo:classif}) that is parameterized by a $\gamma$ parameter.
It has the same overall complexity as \knn{}.
The first step to classify a query $\x$ is to find its $k$ nearest negative neighbors and its $k$ nearest positive neighbors.
Then, the distances to the positive neighbors are multiplied by $\gamma$, to obtain $d_\gamma$.
These $2k$ neighbors are then ranked and the $k$ closest ones are used for classification (with a majority vote, as in \knn{}).
It should be noted that, although $d_{\gamma}$ does not define a proper distance, we can still use any existing fast nearest neighbor search algorithm, because the actual search is done (twice but) only using the original distance $d$.

\vspace{-0.2cm}

\begin{figure}[t]
  \centerline{
    \includegraphics[scale=0.33]{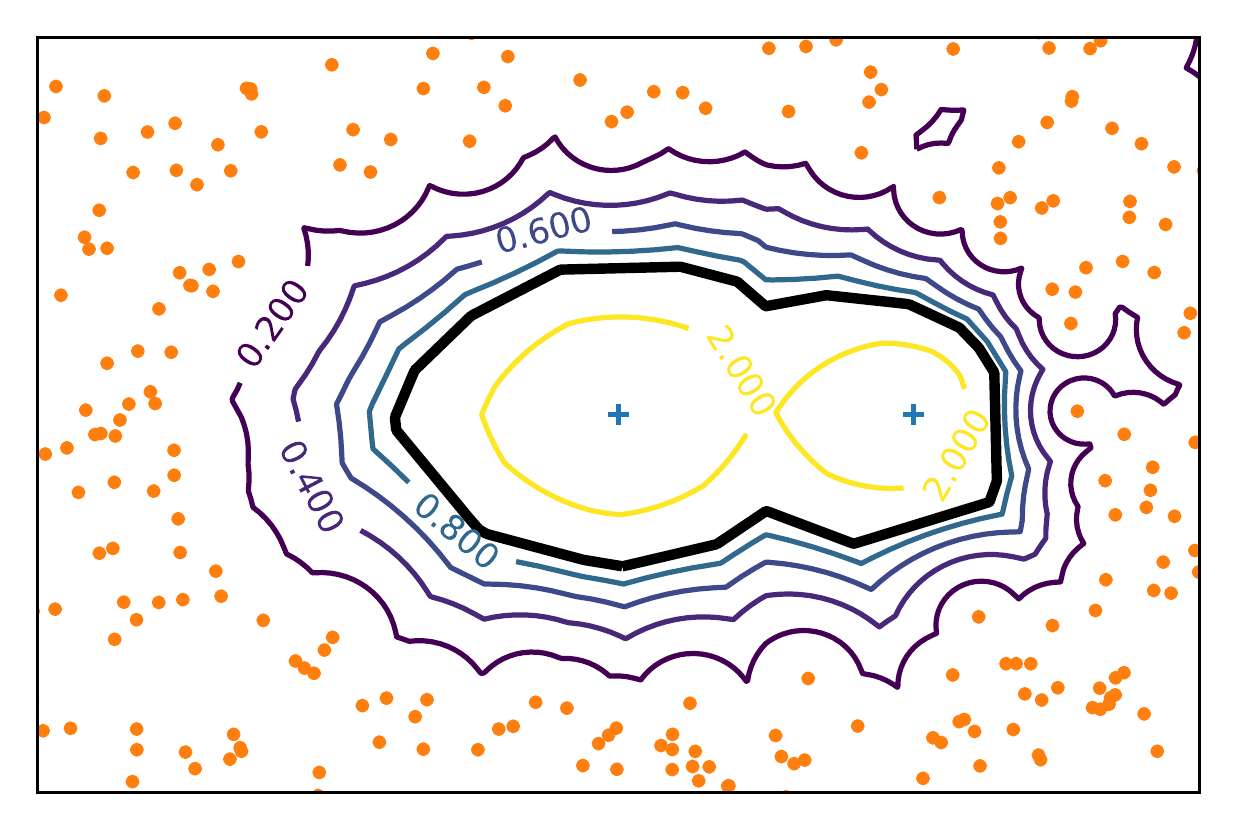}
    \includegraphics[scale=0.33]{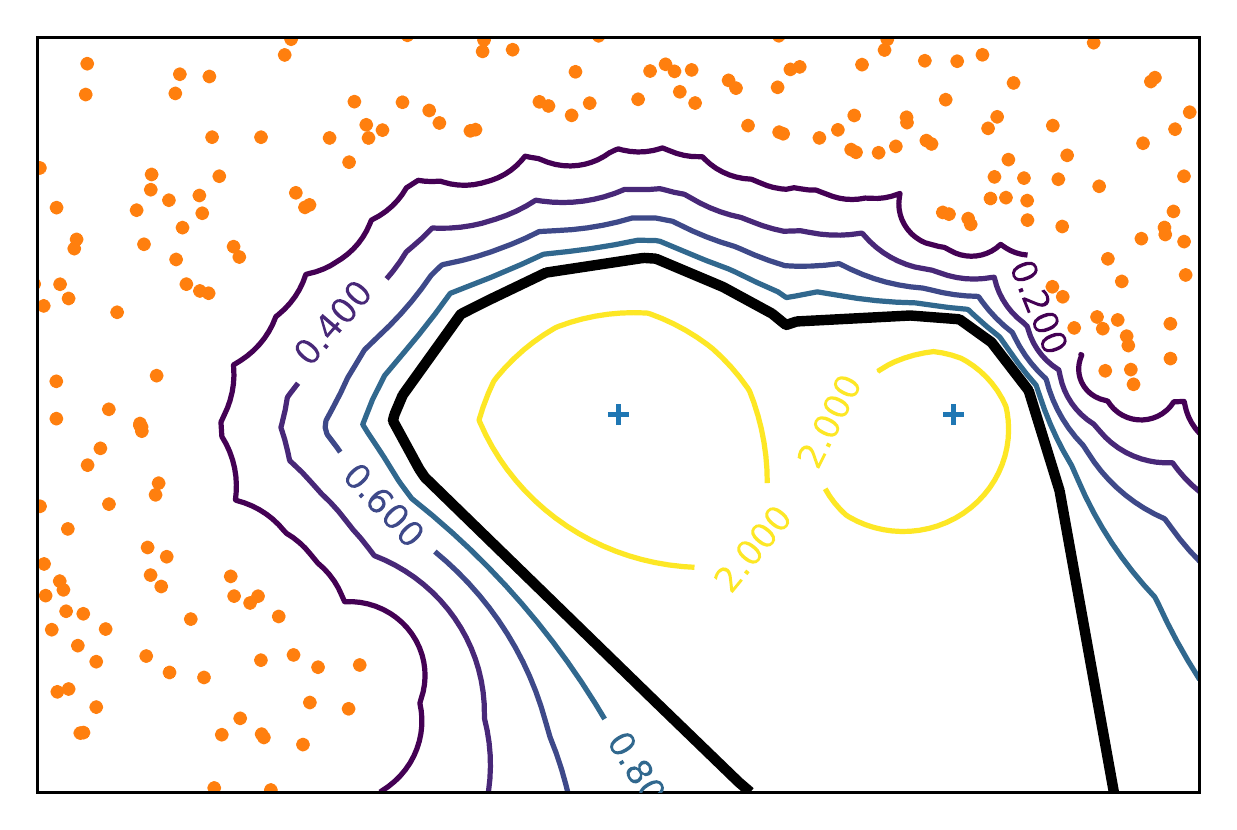}}
\caption{Behavior of the decision boundary according to the $\gamma$ value for the 1-NN classifier on two toy datasets. 
The positive points are represented by blue crosses and the negatives by orange points.
The black line represents the standard decision boundary for the 1-NN classifier, i.e. when $\gamma=1$.}
\label{intro:fig:simuldata}
\end{figure}
\vspace{-0.2cm}

\begin{algorithm}[htb]
\DontPrintSemicolon
\SetAlgoLined
\SetKwInOut{Input}{Input}
\SetKwInOut{Output}{Output}
\Input{a query $\x$ to be classified, a set of labeled samples $S = S_+ \cup S_- $, a number of neighbors $k$, a positive real value $\gamma$, a distance function $d$}
\Output{the predicted label of $\x$}
\BlankLine
$\mathcal{NN}^-,\mathcal{D}^- \leftarrow nn(k, \x, S_-)$ ~~~// nearest negative neighbors with their distances\;
$\mathcal{NN}^+,\mathcal{D}^+ \leftarrow nn(k, \x, S_+)$ ~~~// nearest positive neighbors with their distances\;
$\mathcal{D}^+ \leftarrow \gamma \cdot \mathcal{D}^+$\;
$\mathcal{NN}_\gamma \leftarrow firstK\left(k, sortedMerge( (\mathcal{NN}^-,\mathcal{D}^-), (\mathcal{NN}^+,\mathcal{D}^+))\right)$\; 
$y \leftarrow + $ if $ \left| \mathcal{NN}_\gamma \cap \mathcal{NN}^+ \right| \ge \frac{k}{2} $ else $-$ ~~~// majority vote based on $\mathcal{NN}_\gamma$ \;
\Return $y$\;
\caption{Classification of a new example with $\gamma$\knn{}}
\label{algo:classif}
\end{algorithm}

\subsection{Theoretical analysis}
\label{subsec:analysis}

In this section, we formally analyze what could be a good range of values for the $\gamma$ parameter of our corrected version of the $k-$NN algorithm. 
To this aim, we study what impact $\gamma$ has on the probability to get a false positive (and false negative) at test time and explain why it is important to choose $\gamma<1$ when the imbalance in the data is significant. The following analysis is made for $k=1$ but note that the conclusion still holds for a $k$-NN.

\begin{myprop}{(False Negative probability)} \label{prop:FN}
Let $d_\gamma(\x,\x_+)=\gamma d(\x,\x_+)$, $\forall \gamma >0$, be our modified distance used between a query $\x$ and any positive training example $\x_+$, where $d(\x,\x_+)$ is some distance function. 
Let $FN_\gamma(\z)$ be the probability for a positive example $\z$ to be a false negative using Algorithm (\ref{algo:classif}). 
The following result holds: if $\gamma \leq 1$,
\begin{eqnarray*}
FN_\gamma(\z) & \leq & FN(\z)
\end{eqnarray*}
\end{myprop}

\begin{proof}{(sketch of proof)}
  Let $\epsilon$ be the distance from $\z$ to its nearest-neighbor $N_\z$.
  $\z$ is a false negative if $N_\z \in S_-$ that is all positives $\x' \in S_+$ are outside the sphere $\mathcal{S}_{\frac{\epsilon}{\gamma}}(\z)$ centered at $\z$ of radius $\frac{\epsilon}{\gamma}$. 
  Therefore,
\begin{eqnarray}
FN_{\gamma}(\z)&= &\displaystyle \prod_{\x' \in S_+} \left ( 1-P(\x' \in \mathcal{S}_{\frac{\epsilon}{\gamma}}(\z)) \right ), \nonumber\\
&= & \left ( 1-P(\x' \in \mathcal{S}_{\frac{\epsilon}{\gamma}}(\z)) \right)^{m_+} \label{eq:FN_g}
\end{eqnarray}
while 
\begin{eqnarray}
FN(\z)=\left( 1-P(\x' \in \mathcal{S}_\epsilon(\z)) \right)^{m_+}.    \label{eq:FN}
\end{eqnarray}
Solving (\ref{eq:FN_g}) $\leq$ (\ref{eq:FN}) implies  $\gamma \leq 1$.
\end{proof}

This result means that satisfying $\gamma < 1$ allows us to increase the decision boundary around positive examples (as illustrated in Fig. \ref{intro:fig:simuldata}), yielding a smaller risk to get false negatives at test time.  
An interesting comment can be made from Eq.(\ref{eq:FN_g}) and (\ref{eq:FN}) about their convergence. 
As $m_+$ is supposed to be very small in imbalanced datasets, the convergence of  $FN(\z)$ towards 0 is pretty slow, while one can speed-up this convergence with $FN_\gamma(\z)$ by increasing the radius of the sphere $\mathcal{S}_{\frac{\epsilon}{\gamma}}(\z)$, that is taking a small value for $\gamma$.

\begin{myprop}{(False Positive probability)} \label{prop:FP}
Let $FP_\gamma(\z)$ be the probability for a negative example $\z$ to be a false positive using Algorithm (\ref{algo:classif}). The following result holds: if $\gamma \geq 1$,
\begin{eqnarray*}
FP_\gamma(\z) & \leq & FP(\z)
\end{eqnarray*}
\end{myprop}

\begin{proof}{(sketch of proof)}
 Using the same idea as before, we get: 
\begin{eqnarray}
FP_{\gamma}(\z) & = & \displaystyle \prod_{\x' \in S_-} \left ( 1-P(\x' \in \mathcal{S}_{\gamma\epsilon}(\z)) \right ), \nonumber \\
& = & \left ( 1-P(\x' \in \mathcal{S}_{\gamma\epsilon}(\z)) \right)^{m_-} \label{eq:FP_g}
\end{eqnarray}
while 
\begin{eqnarray}
FP(\z)= \left ( 1-P(\x' \in \mathcal{S}_{\epsilon}(\z)) \right )^{m_-}. \label{eq:FP}
\end{eqnarray}
Solving (\ref{eq:FP_g}) $\leq$ (\ref{eq:FP}) implies  $\gamma \geq 1$.
\end{proof}

As expected, this result suggests to take $\gamma > 1$ to increase the distance $d_\gamma(\z,\x_+)$ from a negative test sample $\z$ to any positive training example $\x_+$ and thus reduce the risk to get a false positive.  
It is worth noticing that while the two conclusions from Propositions \ref{prop:FN} and \ref{prop:FP} are contradictory, the convergence of $FP_{\gamma}(\z)$ towards 0 is much faster than that of $FN_{\gamma}(\z)$ because $m_- >> m_+$ in an imbalance scenario. Therefore,  fulfilling the requirement $\gamma > 1$ is much less important than satisfying $\gamma < 1$. For this reason, we will impose our Algorithm (\ref{algo:classif}) to take $\gamma \in ]0,1[$. As we will see in the experimental section, the more imbalance the datasets, the smaller the optimal $\gamma$, confirming the previous conclusion.


\section{Experiments}
\label{sec:experiments}

In this section, we present an experimental evaluation of our method on public and real private datasets with comparisons to classic distance-based methods and state of the art sampling strategies able to deal with imbalanced data.
All results are reported using $k=3$. 
Note that if the theoretical study is presented for $k=1$, the same Analysis can be conducted for other values of $k$.
Furthermore, we have decided to present the results for $k=3$ as it is the most used $k$-value for that kind of algorithms (e.g. for LMNN~\cite{weinberger2009distance})
The results for $k=1$ are comparable as the presented results in this section


\subsection{Experimental setup}
\label{subsec:setting}

\begin{table}[t]
\caption{\label{tbl:data} Information about the studied datasets sorted by imbalance ratio. The first part refers to the public datasets, the second one describes the \textit{DGFiP} private datasets.}
\begin{sc}
  \centerline{\scalebox{0.86}{\begin{tabular}{lccccc} 
  \hline
datasets    & size & dim & \%$+$ & \%$-$ & IR  \\   
\hline 
balance      &  625 &  4 & 46.1 & 53.9 &   1.2 \\
autompg      &  392 &  7 & 37.5 & 62.5 &   1.7 \\
ionosphere   &  351 & 34 & 35.9 & 64.1 &   1.8\\
pima         &  768 &  8 & 34.9 & 65.1 &   1.9\\
wine         &  178 & 13 & 33.1 & 66.9 &   2  \\
glass        &  214 &  9 & 32.7 & 67.3 &   2.1 \\ 
german       & 1000 & 23 & 30   & 70   &   2.3 \\
vehicle      &  846 & 18 & 23.5 & 76.5 &   3.3 \\ 
hayes        &  132 &  4 & 22.7 & 77.3 &   3.4 \\
segmentation & 2310 & 19 & 14.3 & 85.7 &   6  \\
abalone8     & 4177 & 10 & 13.6 & 86.4 &   6.4 \\
yeast3       & 1484 &  8 & 11   & 89   &   8.1 \\
pageblocks   & 5473 & 10 & 10.2 & 89.8 &   8.8 \\
satimage     & 6435 & 36 &  9.7 & 90.3 &   9.3 \\
libras       &  360 & 90 &  6.7 & 93.3 &  14 \\
wine4        & 1599 & 11 &  3.3 & 96.7 &  29.2\\
yeast6       & 1484 &  8 &  2.4 & 97.6 &  41.4  \\
abalone17    & 4177 & 10 &  1.4 & 98.6 &  71.0 \\ 
abalone20    & 4177 & 10 &  0.6 & 99.4 & 159.7 \\ 
 \hhline{======}
dgfip 19 2   &  16643 & 265 & 35.1 & 64.9 & 1.9\\ 
dgfip 9 2    &  440 & 173& 24.8 & 75.2 &   3\\ 
dgfip 4 2    &  255 &  82& 20.8 & 79.2 &   3.8 \\
dgfip 8 1    & 1028 & 255& 17.8 & 82.2 &   4.6\\
dgfip 8 2    & 1031 & 254& 17.9 & 82.1 &   4.6\\
dgfip 9 1    &  409 & 171& 16.4 & 83.6 &   5.1\\
dgfip 4 1    &  240 &  76& 16.2 & 83.8 &   5.2\\
dgfip 16 1   &  789 & 162& 10.3 & 89.7 &   8.7\\
dgfip 16 2   &  786 & 164&  9.9 & 90.1 &   9.1\\
dgfip 20 3   & 17584 & 294 & 5 & 95 & 19 \\
dgfip 5 3   & 19067 & 318 & 3.9 & 96.1 & 24.9 \\
\hline
  \end{tabular}
  }}
\end{sc}
\end{table}

\begin{table*}[t]
\caption{\label{tbl:res_3NN_full} Results for $3-$NN on the public datasets. The values correspond to the mean F-measure $F_1$ over $5$ runs. The standard deviation is indicated between brackets. The best result on each dataset is indicated in bold. }
\begin{footnotesize}
\begin{center}
\begin{sc}
\scalebox{1}{
  \begin{tabular}{|l|c|c|c|c|c|c|c|} \hline
datasets & 	 $3-$NN & dup$k-$NN& w$k-$NN& cw$k-$NN & kRNN & LMNN & $\gamma k-$NN \\
	 \hline
balance      & 0.954\stdev{0.017} & 0.954\stdev{0.017} & 0.957\stdev{0.017} & 0.961\stdev{0.010} & {\bf0.964}\stdev{0.010} & 0.963\stdev{0.012} & 0.954\stdev{0.029} \\
autompg      & 0.808\stdev{0.077} & 0.826\stdev{0.033} & 0.810\stdev{0.076} & 0.815\stdev{0.053} & {\bf0.837}\stdev{0.040} & 0.827\stdev{0.054} & 0.831\stdev{0.025} \\
ionosphere   & 0.752\stdev{0.053} & 0.859\stdev{0.021} & 0.756\stdev{0.060} & 0.799\stdev{0.036} & 0.710\stdev{0.052} & 0.890\stdev{0.039} & {\bf0.925}\stdev{0.017} \\
pima         & 0.500\stdev{0.056} & 0.539\stdev{0.033} & 0.479\stdev{0.044} & 0.515\stdev{0.037} & {\bf0.579}\stdev{0.055} & 0.499\stdev{0.070} & 0.560\stdev{0.024} \\
wine         & 0.881\stdev{0.072} & 0.852\stdev{0.057} & 0.881\stdev{0.072} & 0.876\stdev{0.080} & 0.861\stdev{0.093} & {\bf0.950}\stdev{0.036} & 0.856\stdev{0.086} \\
glass        & 0.727\stdev{0.049} & 0.733\stdev{0.061} & 0.736\stdev{0.052} & 0.717\stdev{0.055} & 0.721\stdev{0.031} & 0.725\stdev{0.048} & {\bf0.746}\stdev{0.046} \\
german       & 0.330\stdev{0.030} & 0.449\stdev{0.037} & 0.326\stdev{0.030} & 0.344\stdev{0.029} & 0.383\stdev{0.048} & 0.323\stdev{0.054} & {\bf0.464}\stdev{0.029} \\
vehicle      & 0.891\stdev{0.044} & 0.867\stdev{0.027} & 0.891\stdev{0.044} & 0.881\stdev{0.021} & 0.879\stdev{0.034} & {\bf0.958}\stdev{0.020} & 0.880\stdev{0.049} \\
hayes        & 0.036\stdev{0.081} & 0.183\stdev{0.130} & 0.050\stdev{0.112} & 0.221\stdev{0.133} & 0.050\stdev{0.100} & 0.036\stdev{0.081} & {\bf0.593}\stdev{0.072} \\
segmentation & 0.859\stdev{0.028} & 0.862\stdev{0.018} & 0.877\stdev{0.028} & 0.851\stdev{0.022} & 0.797\stdev{0.019} & {\bf0.885}\stdev{0.034} & 0.848\stdev{0.025} \\
abalone8     & 0.243\stdev{0.037} & 0.318\stdev{0.013} & 0.241\stdev{0.034} & 0.330\stdev{0.015} & 0.253\stdev{0.041} & 0.246\stdev{0.065} & {\bf0.349}\stdev{0.018} \\
yeast3       & 0.634\stdev{0.066} & 0.670\stdev{0.034} & 0.634\stdev{0.066} & 0.699\stdev{0.015} & {\bf0.723}\stdev{0.021} & 0.667\stdev{0.055} & 0.687\stdev{0.033} \\
pageblocks   & 0.842\stdev{0.020} & 0.850\stdev{0.024} & 0.849\stdev{0.019} & 0.847\stdev{0.029} & 0.843\stdev{0.023} & {\bf0.856}\stdev{0.032} & 0.844\stdev{0.023} \\
satimage     & 0.454\stdev{0.039} & 0.457\stdev{0.027} & 0.454\stdev{0.039} & 0.457\stdev{0.023} & 0.458\stdev{0.033} & {\bf0.487}\stdev{0.026} & 0.430\stdev{0.008} \\
libras       & 0.806\stdev{0.076} & 0.788\stdev{0.187} & 0.806\stdev{0.076} & 0.789\stdev{0.097} & {\bf0.810}\stdev{0.056} & 0.770\stdev{0.027} & 0.768\stdev{0.106} \\
wine4        & 0.031\stdev{0.069} & {\bf0.090}\stdev{0.086} & 0.031\stdev{0.069} & 0.019\stdev{0.042} & 0.000\stdev{0.000} & 0.000\stdev{0.000} & {\bf0.090}\stdev{0.036} \\
yeast6       & 0.503\stdev{0.302} & 0.449\stdev{0.112} & 0.502\stdev{0.297} & 0.338\stdev{0.071} & 0.490\stdev{0.107} & 0.505\stdev{0.231} & {\bf0.553}\stdev{0.215} \\
abalone17    & 0.057\stdev{0.078} & {\bf0.172}\stdev{0.086} & 0.057\stdev{0.078} & 0.096\stdev{0.059} & 0.092\stdev{0.025} & 0.000\stdev{0.000} & 0.100\stdev{0.038} \\
abalone20    & 0.000\stdev{0.000} & 0.000\stdev{0.000} & 0.000\stdev{0.000} & {\bf0.067}\stdev{0.038} & 0.000\stdev{0.000} & 0.057\stdev{0.128} & 0.052\stdev{0.047} \\
        \hhline{========}
mean         & 0.543\stdev{0.063} & 0.575\stdev{0.053} & 0.544\stdev{0.064} & 0.559\stdev{0.046} & 0.550\stdev{0.041} & 0.560\stdev{0.053} & {\bf0.607}\stdev{0.049} \\
\hline
  \end{tabular}
  }
\end{sc}
\end{center}
\end{footnotesize}
\end{table*}

For the experiments, we use several public datasets from the classic UCI~\footnote{\url{https://archive.ics.uci.edu/ml/datasets.html}} and KEEL~\footnote{\url{https://sci2s.ugr.es/keel/datasets.php}} repositories. 
We also use eleven real fraud detection datasets  provided by the General Directorate of Public Finances (DGFiP) which is part of the French central public administration related to the French Ministry for the Economy and Finance.  These private datasets correspond to data coming from  tax and VAT declarations of French companies and are used for tax fraud detection purpose covering declaration of over-valued, fictitious or prohibited charges, wrong turnover reduction or particular international VAT frauds such as "VAT carousels". The DGFiP performs about 50,000 tax audits per year within a panel covering more than 3,000,000 companies. Being able to select the right companies to control each year is a crucial issue with a potential high societal impact. Thus, designing efficient imbalance learning methods is key. The main properties of the datasets are summarized in Table~\ref{tbl:data}, including the imbalance ratio (IR).

All the datasets are normalized using a min-max normalization such that each feature lies in the range $[-1,1]$.
We randomly draw 80\%-20\% splits of the data to generate the training and test sets respectively. 
Hyperparameters are tuned with a 10-fold cross-validation over the training set. 
We repeat the process over 5 runs and average the results in terms of F-measure $F_1$.
In a first series of experiments, we compare our method, named $\gamma k-$NN,  to 6 other distance-based baselines:
\begin{enumerate}
\item[•] the classic  $k-$Nearest Neighbor algorithm ($k-$NN),
\item[•] the weighted version of $k-$NN using the inverse distance as a weight to predict the label (w$k-$NN)~\cite{dudani1976distance},
\item[•] the class weighted version of $k-$NN (cw$k-$NN)~\cite{barandela2003strategies},
\item[•] the $k-$NN version where each positive is duplicated according to the IR of the dataset (dup$k-$NN),
\item[•] $kR$NN where the sparsity of minority examples is taken into account~\cite{zhang2017krnn} by modifying the way the posterior probability of belonging to the positive class is computed.
\item[•] the  metric learning method LMNN~\cite{weinberger2009distance}. 
\end{enumerate}
Note that we do not compare with~\cite{hajizadeh2014nearest} as the following results are given with $k=3$ while their algorithm can be used only with $k=1$.

We set the number of nearest neighbors to $k=3$ for all  methods. 
The hyperparameter $\mu$ of \textit{LMNN}, weighting the impact of impostor constraints (see~\cite{weinberger2009distance} for more details), is tuned in the range $[0,1]$ using a step of $0.1$. 
Our $\gamma$ parameter is tuned in the range $[0,1]$\footnote{We experimentally noticed that using a larger range for $\gamma$ leads in fact to a potential decrease of performances due to overfitting phenomena. 
This behavior is actually in line with the analysis provided in Section~\ref{subsec:analysis}.} using a step of $0.1$. For $kR$NN, we have used parameters values as described in~\cite{zhang2017krnn}, however we take $k=3$ instead of $1$.

In a second series of experiments, we compare our method to the five oversampling strategies described in Section~\ref{subsec:sampling}: SMOTE, Borderline-SMOTE, ADASYN, SMOTE with ENN, SMOTE with Tomek's link. 
The number of generated positive examples is tuned over the set of ratios $\dfrac{m_+}{m_-}\in \{0.1,0.2,...,0.9,1.0\}$ and such that  the new ratio is greater than the original one before sampling.  
Other parameters of these methods are the default ones used by the package \textit{ImbalancedLearn} of \textit{Scikit-learn}.


\subsection{Results}
\label{subsec:results}

The results on the public datasets using distance-based methods are provided in Table~\ref{tbl:res_3NN_full}.
Overall, our $\gamma k-$NN approach performs much better than its competitors by achieving an improvement of at least $3$ points on average, compared to the 2\textsuperscript{nd} best method ({\sc dup}$k-$NN). The different $k-$NN versions fail globally to provide  models efficient whatever the imbalance ratio. The metric learning approach LMNN is competitive when IR is smaller than 10  (although algorithmically more costly). Beyond, it faces some difficulties to find a relevant projection space due to the lack of positive data. The efficiency of $\gamma k-$NN is not particularly sensitive to the imbalance ratio.

The results for our second series of experiments, focusing on sampling strategies, are reported on Fig.~\ref{exp:fig:compare_sampling}.
We compare  each of the 5 sampling methods with the average performances of $3-$NN and $\gamma k$-NN obtained over the 19 public datasets reported in Table~\ref{tbl:res_3NN_full}. 
Additionally,  we also use $\gamma k-$NN on the top of the sampling methods to evaluate how both strategies are complementary. However, in this scenario, we propose to learn a different $\gamma$ value to be used with the synthetic positives. Indeed, some of them may be generated in some true negative areas and in this situation it might be more appropriate to decrease the influence of such synthetic examples. The $\gamma$ parameter for these examples is then tuned in the range $[0,2]$ using a step of 0.1. 
If one can  easily observe that all the oversampling strategies improve the classic $k-NN$, none of them is better than our $\gamma k$-NN method showing that our approach is able to deal efficiently with imbalanced data. Moreover, we are able to improve the efficiency of $\gamma k$-NN when it is coupled with an oversampling strategy. The choice of the oversampler does not really influence the results. The gains obtained by using a sampling method with $\gamma k$-NN for each dataset is illustrated in Fig.~\ref{exp:fig:compare_best_sampling} (top).

To study the influence of using two $\gamma$ parameters when combined with an oversampling strategy, we show an illustration (Fig. \ref{exp:fig:gamma_sampling} (top)) of the evolution of the $F$-measure with respect to the $\gamma$ values for synthetic and real positive instances.
The best $F$-measure is achieved when the $\gamma$ on real positives is smaller than 1 and when the $\gamma$ on synthetic positives is greater than 1, justifying the interest of using two parameterizations of $\gamma$.
In Fig. \ref{exp:fig:gamma_sampling} (bottom), we show how having two $\gamma$ values gives the flexibility to independently control the increased influence of real positives and the one of artificial positives.

\begin{figure}[t]
\centerline{\includegraphics[scale=0.53]{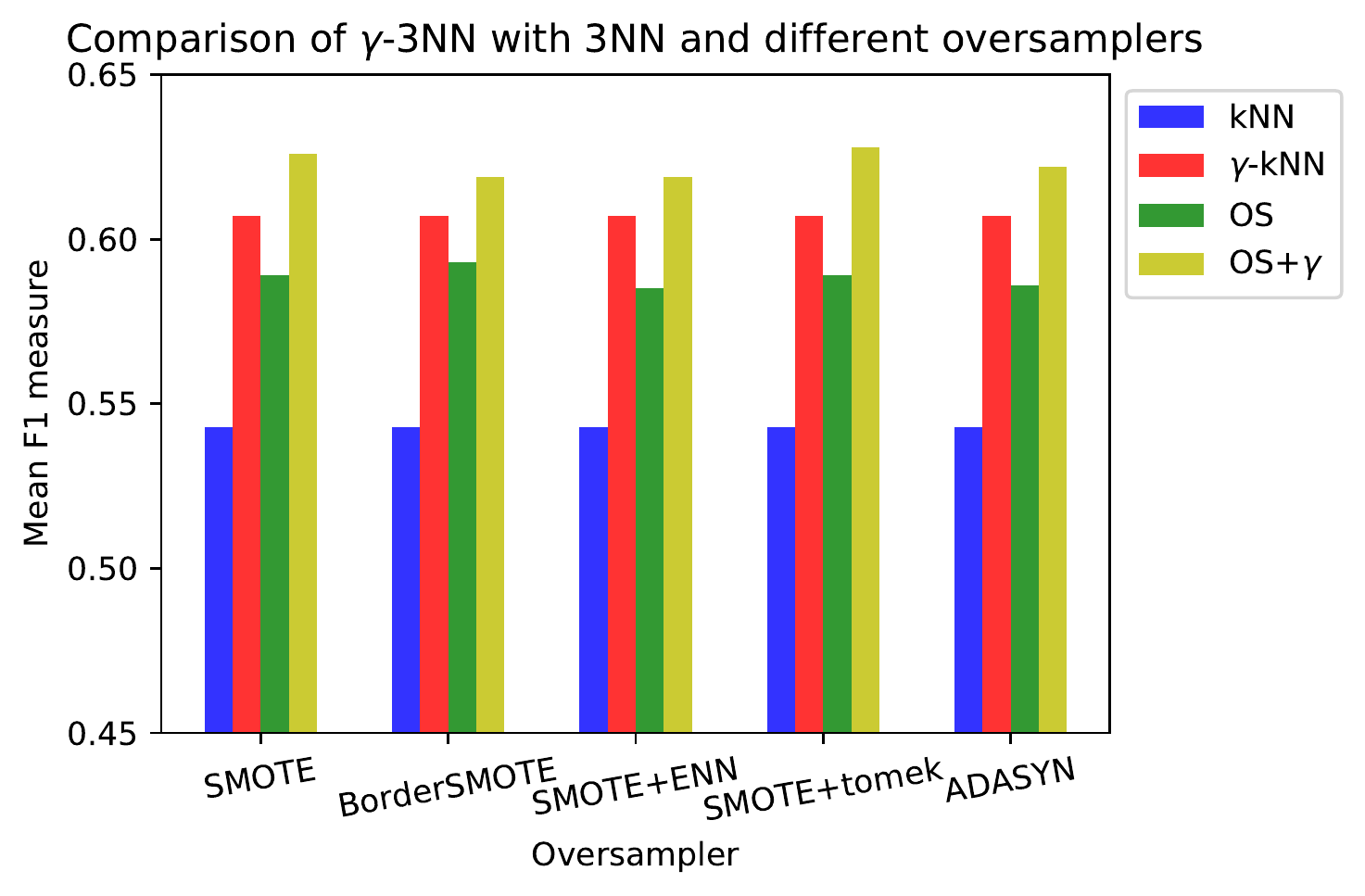}}
\caption{Comparison of different sampling strategies averaged over the 19 public datasets.  $OS$ refers to the results of the corresponding sampling strategy and $OS+\gamma$ to the case when the sampling strategy is combined with $\gamma k$-NN. $k-$NN and $\gamma k-$NN refers to the results of these methods without oversampling as obtained in Table~\ref{tbl:res_3NN_full}.
(numerical values for these graphs are provided in supplementary material) }
\label{exp:fig:compare_sampling}
\end{figure}

\begin{figure}[t]
\centerline{\includegraphics[trim= 10 50 0 50, clip, scale=0.55]{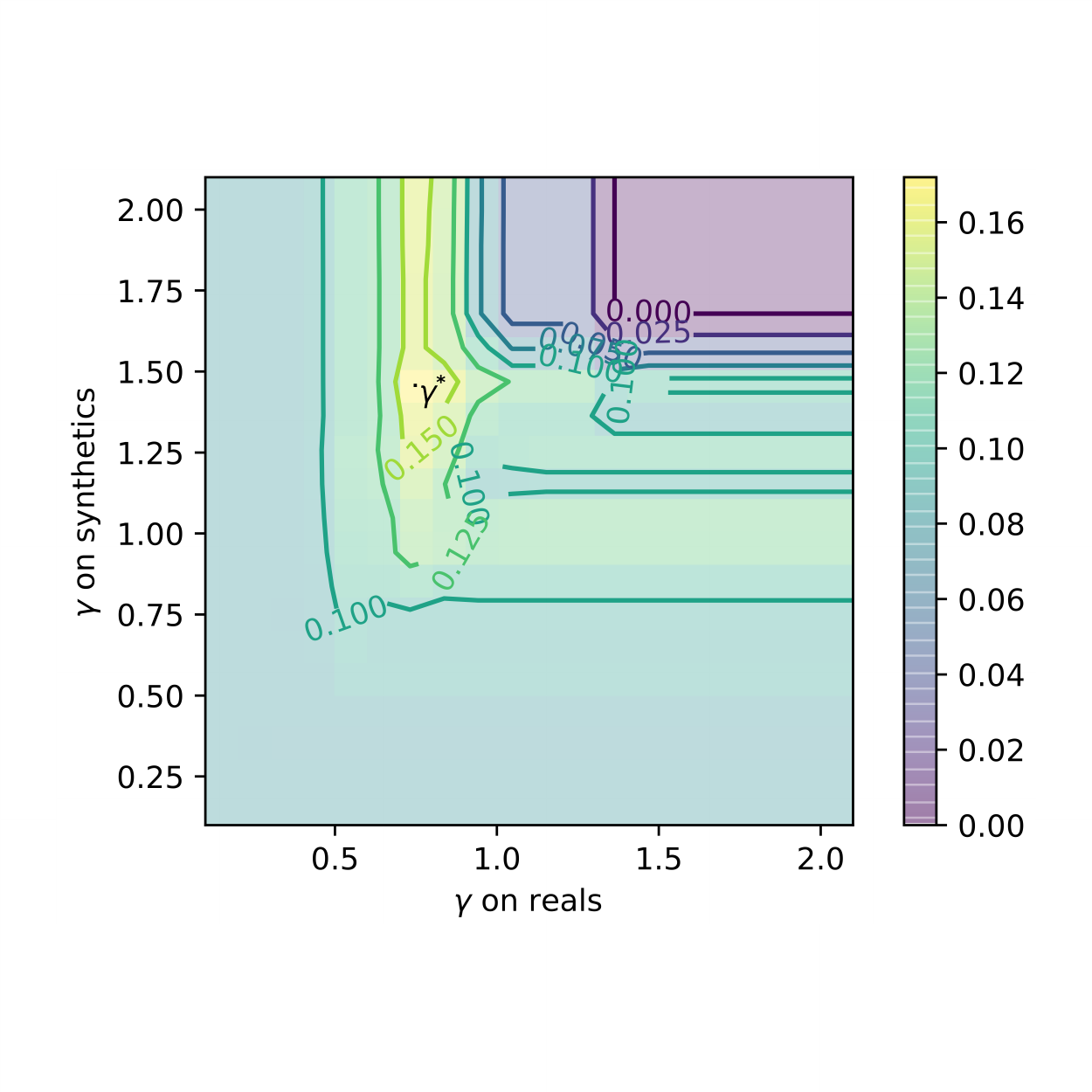}}
\centerline{\includegraphics[scale=0.45]{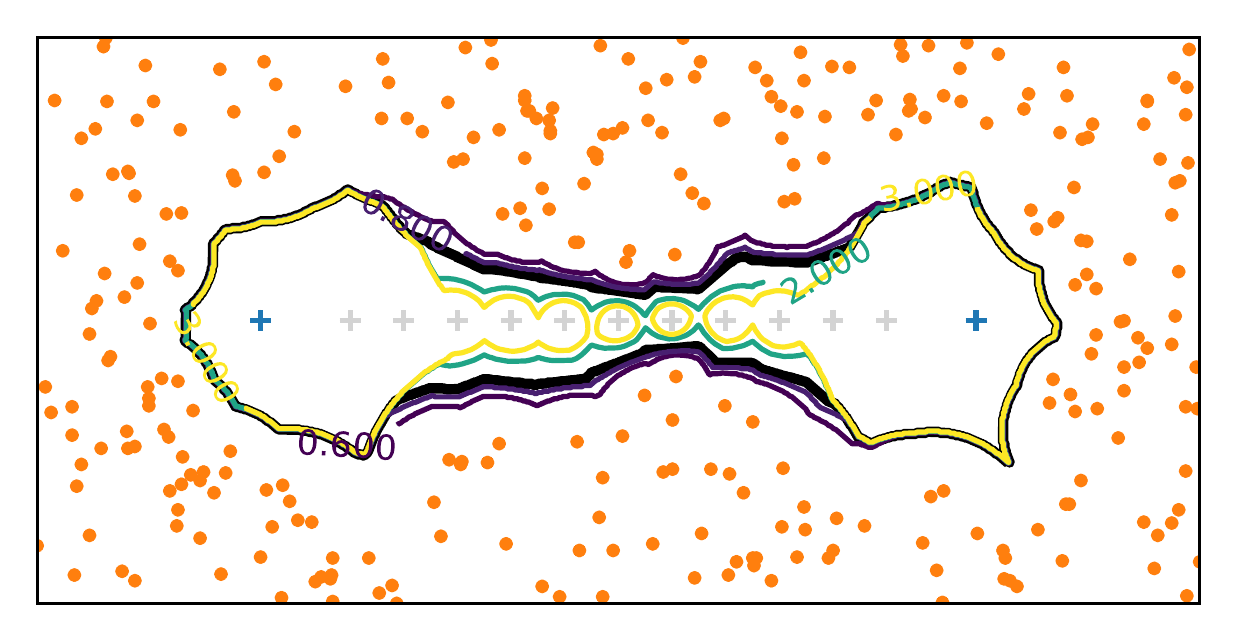}}
  \caption{(Top) An example of heatmap that shows the best couple of $\gamma$ for the OS+$\gamma$\knn{} strategy on the yeast6 dataset with SMOTE and Tomek's link.
  (Bottom) Illustration, on a toy dataset, of the effect of varying the $\gamma$ for generated positive points (in grey) while keeping a fixed $\gamma=0.4$ for real positive points. }
\label{exp:fig:gamma_sampling}
\end{figure}

We now propose a study on the influence of the imbalance ratio on the optimal $\gamma$-parameter. We consider the \textit{Balance} dataset which has the smallest imbalance ratio that we increase  by iteratively randomly under-sampling  the minority class over the training set. We report the results on  Fig.~\ref{exp:fig:compare_best_sampling} (bottom). As expected, we can observe that the optimal $\gamma$ value decreases when the imbalance increases. However, note that from a certain IR (around 15), $\gamma$ stops decreasing to be able to keep a satisfactory F-Measure.

\begin{figure}[h!]
\centerline{\includegraphics[scale=0.435]{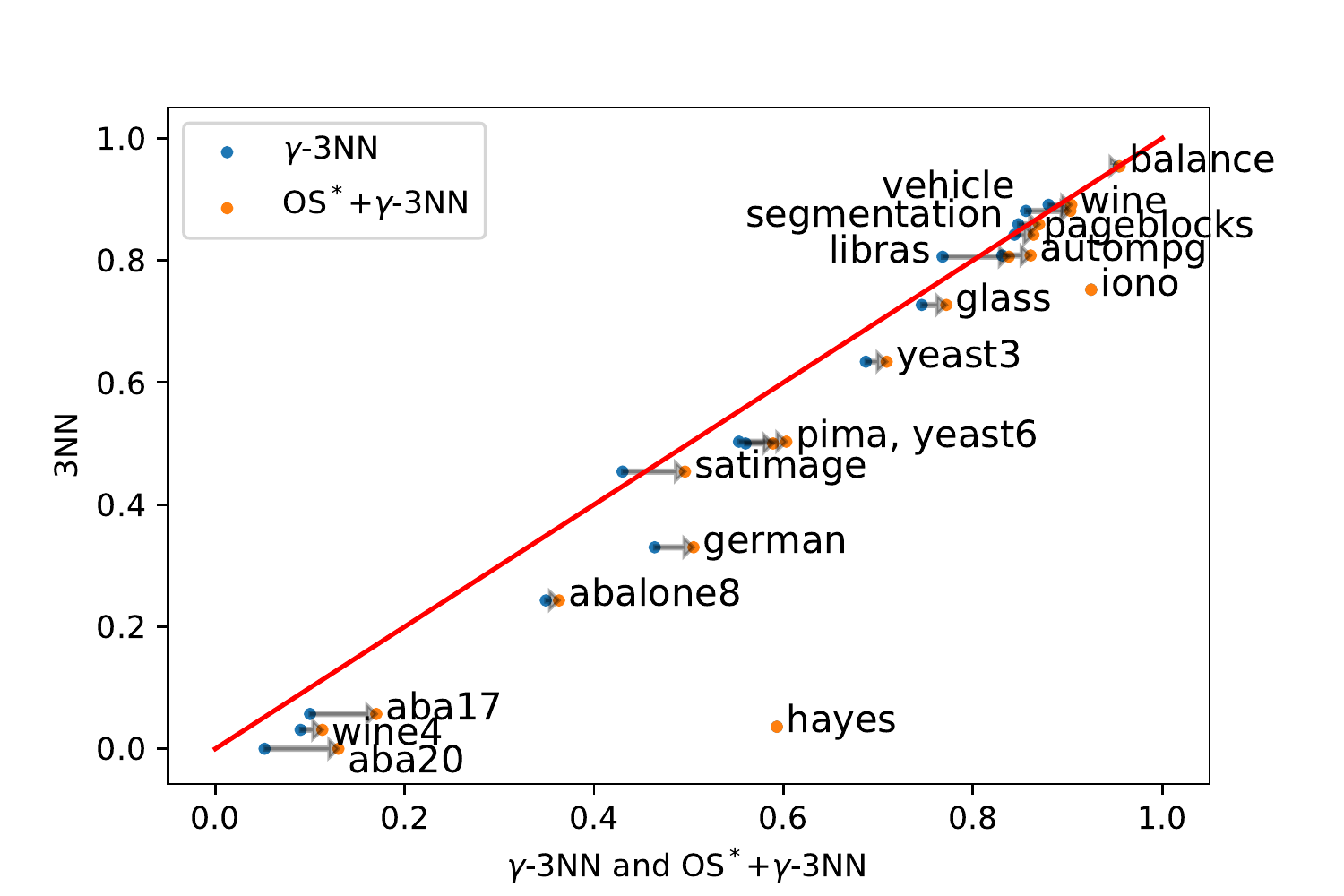}}
\centerline{ \includegraphics[scale=0.435]{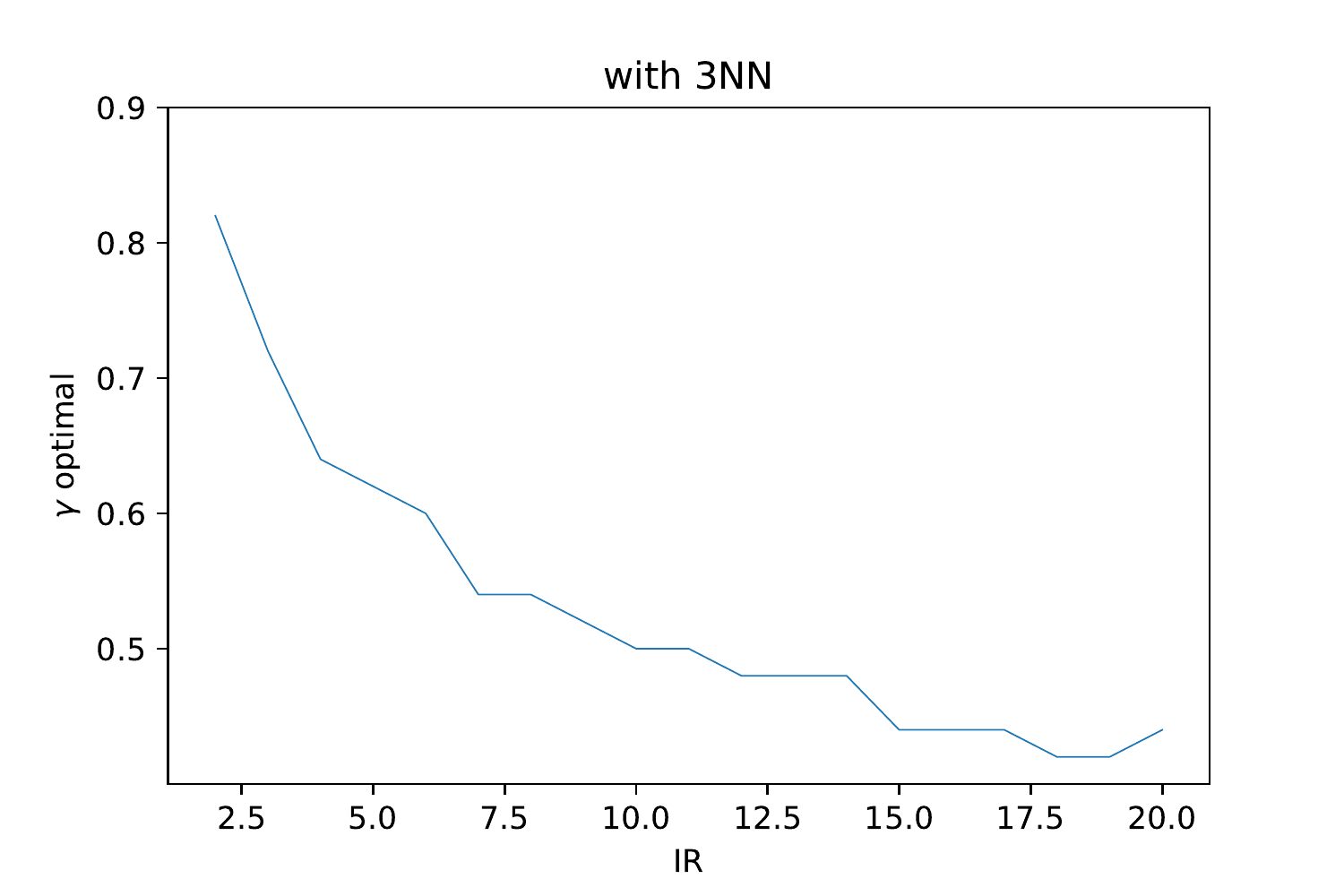}}
\caption{(Top) Comparison of $k$-NN with (i) $\gamma k-$NN (points in blue) and (ii) $\gamma k-$NN coupled with the best sampling strategy (OS$^\star$) (points in orange) for each dataset  and for $k=3$. Points below the line $y=x$ means that $k$-NN is outperformed. (Bottom) Evolution of the optimal $\gamma$ value with respect to the IR for $k=3$.}
\label{exp:fig:compare_best_sampling}
\end{figure}

\begin{table}[tb]
\caption{ \label{tbl:res_3NN_dgfip} Results for $3-$NN on the DGFiP datasets. The values correspond to the mean F-measure $F_1$ over $5$ runs. The best result on each dataset is indicated in bold while the second is underlined.}
\begin{center}
\begin{sc}
\scalebox{0.85}{
  \begin{tabular}{|l|c|c|c|c|c|} \hline
datasets &  $3-$NN & $\gamma k-$NN & SMOTE & SMOTE+$\gamma k-$NN \\
	 \hline
Dgfip19 2 & 0,454\stdev{0,007} & \underline{0,528}\stdev{0,005} & 0,505\stdev{0,010} & {\bf0,529}\stdev{0,003} \\
Dgfip9 2  & 0,173\stdev{0,074} & \underline{0,396}\stdev{0,018} & 0,340\stdev{0,033} & {\bf0,419}\stdev{0,029} \\
Dgfip4 2  & 0,164\stdev{0,155} & \underline{0,373}\stdev{0,018} & 0,368\stdev{0,057} & {\bf0,377}\stdev{0,018} \\
Dgfip8 1  & 0,100\stdev{0,045} & {\bf0,299}\stdev{0,010} & 0,278\stdev{0,043} & {\bf0,299}\stdev{0,011} \\
Dgfip8 2  & 0,140\stdev{0,078} & 0,292\stdev{0,028} & {\bf0,313}\stdev{0,048} & \underline{0,312}\stdev{0,021} \\
Dgfip9 1  & 0,088\stdev{0,090} & 0,258\stdev{0,036} & \underline{0,270}\stdev{0,079} & {\bf0,288}\stdev{0,026} \\
Dgfip4 1  & 0,073\stdev{0,101} & \underline{0,231}\stdev{0,139} & 0,199\stdev{0,129} & {\bf0,278}\stdev{0,067} \\
Dgfip16 1 & 0,049\stdev{0,074} & 0,166\stdev{0,065} & \underline{0,180}\stdev{0,061} & {\bf0,191}\stdev{0,081} \\
Dgfip16 2 & 0,210\stdev{0,102} & 0,202\stdev{0,056} & \underline{0,220}\stdev{0,043} & {\bf0,229}\stdev{0,026} \\
Dgfip20 3 & 0,142\stdev{0,015} & \underline{0,210}\stdev{0,019} & 0,199\stdev{0,015} & {\bf0,212}\stdev{0,019} \\
Dgfip5 3  & 0,030\stdev{0,012} & 0,105\stdev{0,008} & {\bf0,110}\stdev{0,109} & \underline{0,107}\stdev{0,010} \\
        \hhline{======}
mean      & 0,148\stdev{0,068} & \underline{0,278}\stdev{0,037} & 0,271\stdev{0,057} & {\bf0,295}\stdev{0,028} \\
\hline
  \end{tabular}
  }
\end{sc}
\end{center}
\end{table}

The results for the real datasets of the DGFiP are available in Table~\ref{tbl:res_3NN_dgfip}. 
Note that only the SMOTE algorithm is  reported here since the other oversamplers have comparable performances.
The analysis of the results leads to observations similar as the ones made for the public datasets.
Our $\gamma-k$NN approach outperforms classic $k-$NN and is better than the results obtained by the SMOTE strategy.
Coupling the SMOTE sampling method with our distance correction $\gamma k$-NN allows us to improve the global performance showing the applicability of our method on real data.


\section{Conclusion}
\label{sec:conclusion}
In this paper, we have proposed a new strategy that addresses the problem of learning from imbalanced datasets, based on the $k-$NN algorithm and that modifies the distance to the positive examples. 
It has been shown to outperform its competitors in term of F$_1$-measure.
Furthermore, the proposed approach is complementary to oversampling strategies and can even increase their performance.
Our  $\gamma k-$NN algorithm, despite its simplicity, is highly effective even on real data sets.

Two lines of research deserve future investigations.
We can note that tuning $\gamma$ is equivalent to building a diagonal matrix (with $\gamma^2$ in the diagonal) and applying a Mahalanobis distance only between a query and a positive example.
This comment opens the door to a new metric learning algorithm dedicated to optimizing a PSD matrix under F-Measure-based constraints.
If one can learn such a matrix, the second perspective will consist in deriving generalization guarantees over the learned matrix.
In addition, making $\gamma$ non-stationary (a $\gamma(\x)$ that smoothly varies in $\mathcal{X}$) would increase the model flexibility.

\bibliography{ICTAI2019.bib}
\bibliographystyle{apalike}

\end{document}